\documentclass[times]{TRR}

\usepackage{pgfplots,lipsum}
\usepackage{tikz,graphicx}
\usepackage{authblk}
\usepackage{pgf,tikz}
\usepackage{caption}
\usepackage{booktabs}
\usepackage{geometry}
\usepackage{longtable}
\usepackage{stfloats}  % put this in preamble
\usepackage{comment}

\usetikzlibrary{positioning}
\usetikzlibrary{calc} 
\usepackage{blkarray}
\usepackage{setspace}
\usepackage{moreverb,url}
\newtheorem{thm}{Theorem}[section]
\newtheorem{prop}[thm]{Proposition}
\newtheorem{dfn}[thm]{Definition}
\newtheorem{ex}[thm]{Example}

\newtheorem{cor}[thm]{Corollory}
\newtheorem{rmk}[thm]{Remark}

\usepackage[colorlinks,bookmarksopen,bookmarksnumbered,citecolor=red,urlcolor=red]{hyperref}

\newcommand\BibTeX{{\rmfamily B\kern-.05em \textsc{i\kern-.025em b}\kern-.08em
T\kern-.1667em\lower.7ex\hbox{E}\kern-.125emX}}

\def\volumeyear{2020}

\begin{document}

\runninghead{Shanookha Ali and Nitha Niralda P C}

\title{Identifying Critical Pathways in Coronary Heart Disease via Fuzzy Subgraph Connectivity}

\author{Shanookha Ali\affilnum{1} and Nitha Niralda P C \affilnum{2}}

\affiliation{\affilnum{1}Department of General Science, Birla Institute of Technology \& Science, Pilani, Dubai Campus, Dubai 345055, United Arab Emirates\\
\affilnum{2}Department of Mathematics \& Statistics, Providence Women’s College, Calicut, Kerala, 673009, India}

\corrauth{Shanookha Ali, shanookha@dubai.bits-pilani.ac.in}

\begin{abstract}
Coronary heart disease (CHD) arises from complex interactions among uncontrollable factors, controllable lifestyle factors, and clinical indicators, where relationships are often uncertain. Fuzzy subgraph connectivity (FSC) provides a systematic tool to capture such imprecision by quantifying the strength of association between vertices and subgraphs in fuzzy graphs. In this work, a fuzzy CHD graph is constructed with vertices for uncontrollable, controllable, and indicator components, and edges weighted by fuzzy memberships. Using FSC, we evaluate connectivity to identify strongest diagnostic routes, dominant risk factors, and critical bridges. Results show that FSC highlights influential pathways, bounds connectivity between weakest and strongest correlations, and reveals critical edges whose removal reduces predictive strength. Thus, FSC offers an interpretable and robust framework for modeling uncertainty in CHD risk prediction and supporting clinical decision-making. 
\end{abstract}

\maketitle

\section{Introduction}
\noindent Coronary heart disease (CHD) is a leading global health burden, accounting for significant morbidity and mortality. The identification and analysis of CHD risk factors is a major challenge in preventive cardiology. These risk factors are broadly classified into \emph{uncontrollable factors} such as age and family history, and \emph{controllable factors} such as smoking, diet, and physical activity. Additionally, clinical \emph{indicators} such as blood pressure, cholesterol levels, and electrocardiogram (ECG) findings serve as intermediate measures linking lifestyle and hereditary factors with disease outcomes.\medskip

In the study of fuzzy graphs, various aspects such as connectivity, spanning structures, and network flow have been extensively explored. The foundational concept of fuzzy sets introduced by Zadeh~\cite{adeh} laid the groundwork for the development of fuzzy graph theory, which was formalized by Rosenfeld~\cite{ros1977}. Early investigations on fuzzy graphs and their properties were carried out by Bhattacharya~\cite{bhat1987}, Bhutani and Rosenfeld~\cite{but2003, butarc2023} and Rosenfeld~\cite{ros1977},  who examined strong arcs and fuzzy end nodes.  Applications of vertex and node connectivity in fuzzy graphs have been explored in the context of human trafficking networks by Ali et al.~\cite{shan2018}. Comprehensive treatments of fuzzy graph theory can be found in the books by Mordeson and Nair~\cite{mod2000} and Mordeson et al.~\cite{mathew2018}, including detailed discussions on node and arc connectivity Mathew and Sunitha~\cite{mathew2010}.\medskip 

Vertex connectivity in fuzzy graphs has been analyzed to evaluate resilience and vulnerability in trafficking chains Ali et al.~\cite{shan2018}, while Hamiltonian fuzzy graphs provide a foundation for understanding cyclical structures that often underlie such networks Ali et al.~\cite{ali2021}. More recently, the concept of containers and spanning containers has been introduced to study the flow and concealment strategies of traffickers within uncertain environments Ali et al.~\cite{ali2024}. Complementing these contributions, Mordeson, Mathew, and Ali have applied fuzzy path fans to model the health consequences faced by trafficking victims, thereby highlighting the applicability of fuzzy graph theory beyond structural analysis and into human well-being  by Mordeson et al.~\cite{mode2017}. Together, these works underscore the versatility of fuzzy graph theory in addressing both theoretical challenges and real-world problems associated with trafficking.\medskip

Conventional diagnostic methods, including regression and probabilistic models, often assume precise relationships between variables. However, in real-world medical data, relationships are rarely crisp or deterministic. For example, the effect of age on CHD varies across populations, and the influence of smoking on cardiovascular risk may depend on duration, intensity, and co-occurring conditions. These inherent uncertainties motivate the application of \emph{fuzzy graph theory}, which provides a mathematical framework for handling imprecise, uncertain, and approximate relationships.\medskip

In this study, we construct a \emph{fuzzy CHD graph} in which vertices represent uncontrollable, controllable, and indicator factors, while edges denote their relationships with membership values in $[0,1]$. Using fuzzy connectivity concepts, we investigate:
\begin{enumerate}
    \item pairwise connectivity ($CONN_G(u,v)$) between individual risk factors,
    \item vertex-to-subgraph connectivity ($CONN_G(x,H)$), which evaluates the influence of a single factor on a group of related components, and
    \item subgraph-to-subgraph connectivity ($CONN_G(H_i,H_j)$), which assesses the global strength of interaction between categories of factors.
\end{enumerate}

Our analysis shows that fuzzy connectivity not only quantifies the strength of risk pathways but also identifies \emph{critical bridges} edges whose removal significantly reduces connectivity. Such bridges highlight key clinical factors (e.g., smoking – ECG relationship) that dominate diagnostic predictions. Moreover, strongest paths reveal the most significant diagnostic routes, providing interpretability to clinicians. By bounding connectivity between the weakest and strongest observed correlations, the framework also ensures reliable upper and lower limits for risk prediction.

Overall, this fuzzy graph-based approach contributes to a more nuanced understanding of CHD risk dynamics. By capturing uncertainty and identifying strongest connections, the method supports both diagnostic decision-making and the prioritization of preventive interventions. 

\section{Preliminaries}

In this section, we briefly recall the basic concepts of fuzzy sets, fuzzy graphs, and fuzzy connectivity that will be used in the sequel. Throughout, we follow standard notations used in the literature \cite{ros1977}, \cite{mode2018}, \cite{bhat1987}.

A fuzzy set $A$ on a universe $X$ is defined by a membership function 
$\mu_A : X \to [0,1]$, where $\mu_A(x)$ represents the degree of membership of $x$ in $A$. 
For example, if $X$ is the set of possible ages of patients, then the fuzzy set ``elderly'' may assign $\mu_A(65) = 0.8$, $\mu_A(50)=0.4$, reflecting vagueness in age categorization. 
This concept provides a natural tool to model uncertainty in medical datasets.

A fuzzy graph $G$ is defined as a pair $G = (\sigma, \mu)$, where 
$\sigma : V \to [0,1]$ assigns a membership value to each vertex $v \in V$ 
and $\mu : V \times V \to [0,1]$ assigns a membership value to each edge $(u,v)$, subject to the condition
\[
\mu(u,v) \leq \min\{\sigma(u), \sigma(v)\}.
\]
Here $\sigma(u)$ may represent the relevance of a CHD factor, while $\mu(u,v)$ denotes the strength of relationship between two factors.
This generalization of classical graphs was first introduced by Rosenfeld \cite{ros1977}.
A \emph{path} in a fuzzy graph $G$ is a sequence of vertices $u_0,u_1,\dots,u_k$ such that $\mu(u_i,u_{i+1}) > 0$ for all $i$. 
The \emph{strength} of a path $P$ is defined as
\[
\operatorname{str}(P) = \min_{(u_i,u_{i+1}) \in P} \mu(u_i,u_{i+1}).
\]
The \emph{$u$--$v$ connectivity} in $G$ is then given by
\[
CONN_G(u,v) = \max_{P:u \leadsto v} \operatorname{str}(P),
\]
where the maximum is taken over all paths connecting $u$ and $v$.

Some basic results used in this paper are summarized below:
\begin{itemize}
    \item Symmetry: $CONN_G(H_i,H_j) = CONN_G(H_j,H_i)$.
    \item Bounds: $r(G) \leq CONN_G(H_i,H_j) \leq d(G)$, where $r(G)$ and $d(G)$ are the minimum and maximum edge strengths in $G$.
    \item Bridge property: If $uv$ is a fuzzy bridge in $G$, then there exist subgraphs $H_i,H_j$ such that $CONN_G(H_i,H_j) = \mu(uv)$.
    \item Strongest path: The strongest path between two subgraphs corresponds to the most significant diagnostic route in an applied setting.
\end{itemize}
These concepts will be applied to the fuzzy CHD graph.

\section{Fuzzy subgraph connectivity}
\medskip
\newcommand{\mx}[1]{\displaystyle\max_{#1}}
\begin{dfn}
 $u-v$ connectivity: Maximum of the strengths of all paths between $u$ and $v$ is called strength of connectedness between $u$ and $v$ and denoted by $CONN_G(u,v)$.
\end{dfn}
\begin{dfn}
Let $G=(\sigma, \mu)$ be a fuzzy graph with proper  fuzzy subgraph $H(\nu,\tau)$. For $x$ $\in$ $\sigma^* \backslash \nu^*$, $x-H$ connectivity is defined as maximum of strength of connectedness between $x$ and $u$ where $u \in \nu^*$,  denoted by $CONN_G(x,H)$. That is $$CONN_G(x,H) = \mx{u\in \nu^*} CONN_G(x,u)$$.
\end{dfn}

\begin{figure}[htbp]
\centering
 \scalebox{0.6}{
\begin{tikzpicture}[>=stealth, node distance=1cm, thick]
	\ifx\plotpoint\undefined\newsavebox{\plotpoint}\fi
	\draw (2,2) -- (4,2) -- (4,5) -- (2,5);
	\draw (4,2) -- (6,3.5) -- (4,5);
	\draw (10,2) -- (12,3.5) -- (10,5) --(10,2);
	\draw (2,5) -- (4,2);
	\put(56,56){\circle*{5}}
	\put(115,142){\circle*{5}}
	\put(56,142){\circle*{5}}
	\put(170,99){\circle*{5}}
	\put(115,56){\circle*{5}}
	\put(50,150){\makebox(0,0)[cc]{$a$}}
	\put(50,48){\makebox(0,0)[cc]{$e$}}
	\put(118,151){\makebox(0,0)[cc]{$b$}}
	\put(118,48){\makebox(0,0)[cc]{$d$}}
	\put(179,99){\makebox(0,0)[cc]{$c$}}
	\put(124,100){\makebox(0,0)[cc]{$0.1$}}
	\put(82,150){\makebox(0,0)[cc]{$0.4$}}
	\put(82,45){\makebox(0,0)[cc]{$0.3$}}
	\put(150,130){\makebox(0,0)[cc]{$0.15$}}
	\put(150,70){\makebox(0,0)[cc]{$0.1$}}
	\put(95,105){\makebox(0,0)[cc]{$0.9$}}
	\put(284,142){\circle*{5}}
	\put(342,99){\circle*{5}}
	\put(284,56){\circle*{5}}
	\put(284,151){\makebox(0,0)[cc]{$b$}}
	\put(284,48){\makebox(0,0)[cc]{$d$}}
	\put(350,99){\makebox(0,0)[cc]{$c$}}
	\put(320,130){\makebox(0,0)[cc]{$0.15$}}
	\put(296,100){\makebox(0,0)[cc]{$0.1$}}
	\put(320,68){\makebox(0,0)[cc]{$0.1$}}
	\put(72,20){\makebox(0,0)[cc]{$(a)\,\,\, G$}}
	\put(320,20){\makebox(0,0)[cc]{$(b)\,\,\, H$}}
	
\end{tikzpicture}}
\vspace{1cm}
\caption{ Fuzzy graphs $G$ with proper induced fuzzy subgraph $H$.}
\label{fig1}
\end{figure}
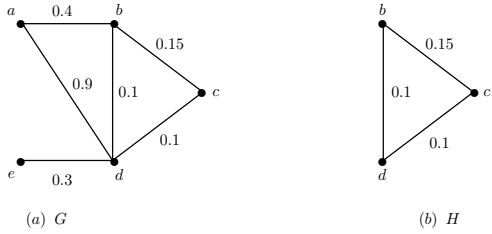

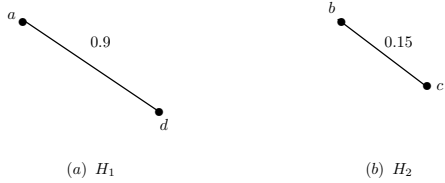
\begin{figure}[htbp]
\centering
 \scalebox{0.6}{
\begin{tikzpicture}[>=stealth, node distance=1cm, thick]
	\ifx\plotpoint\undefined\newsavebox{\plotpoint}\fi
	\draw  (10,2.5) -- (8,4);
	\draw (1,4) -- (4,2);
	\put(29,112){\circle*{5}}
	\put(115,56){\circle*{5}}
	\put(22,116){\makebox(0,0)[cc]{$a$}}
	\put(118,48){\makebox(0,0)[cc]{$d$}}
	\put(78,100){\makebox(0,0)[cc]{$0.9$}}
	\put(230,112){\circle*{5}}
	\put(284,72){\circle*{5}}
	\put(224,121){\makebox(0,0)[cc]{$b$}}
	\put(292,72){\makebox(0,0)[cc]{$c$}}
	\put(266,100){\makebox(0,0)[cc]{$0.15$}}
	\put(72,20){\makebox(0,0)[cc]{$(a)\,\,\,H_1$}}
	\put(260,20){\makebox(0,0)[cc]{$(b)\,\,\, H_2$}}
	
\end{tikzpicture}}
\vspace{1cm}

\caption{ Proper induced fuzzy subgraphs of fuzzy graph $G$ in Example~\ref{ex1.3}}
\label{fig2}
\end{figure}

\begin{ex}\label{ex1.3}
 Let $G=(\sigma, \mu)$ be with $\sigma^* = \{a,b,c,d,e\}$, $\sigma(a,b)= 0.4$, $\sigma(b,c) = 0.15$, $\sigma(b,d) = \sigma(c,d) = 0.1$, $\sigma(a,d)= 0.9$ and $\sigma(e,d) = 0.3$ (see Figure \ref{fig1} $(a)$). Let $H=(\nu, \tau)$ be fuzzy subgraph induced by $\nu^* = \{b, c,d \}$ (see Figure \ref{fig1} $(b)$). For $a \in \sigma^*\backslash \nu^*$, $CONN_G(a,H) = \max \{CONN_G(a,b), CONN_G(a,d), CONN_G(a,c) \}$ $=\max \{0.4,0.15,0.9\} = 0.9$. Similarly   $CONN_G(e,H) = 0.4$.
\end{ex}

\begin{dfn}
 Let $G=(\sigma, \mu)$ be a fuzzy graph with proper disjoint induced fuzzy subgraphs $H_1$ and $H_2$. Then fuzzy subgraph connectivity (FSC) between $H_1$ and $H_2$   denoted by $CONN_G(H_1,H_2)$ is maximum of $CONN_G(x,H_2)$, for $x \in  \sigma_1^*$. That is  $CONN_G(H_1,H_2) = \mx{x \in \sigma_1^*} CONN_G(x,H_2)$.
\end{dfn}

\noindent Consider the  fuzzy graph $G$ in Figure \ref{fig1} $(a)$. Let $H_1$ and  $H_2$  be fuzzy subgraphs of $G$ induced by $\{a,d\}$ and $\{b,c\}$ respectively (see Figure \ref{fig2} $(a)$ and $(b))$. Then $CONN_G(H_1,H_2) = \max \{ CONN_G(a,H_2), $ $CONN_G(d,H_2)\}  = 0.4$. \medskip

%\noindent Fuzzy subgraph connectivity have some interesting algebraic properties.\medskip
\begin{prop}\label{prop1.5}
	 Let $G=(\sigma, \mu)$ be a fuzzy graph with proper disjoint induced fuzzy subgraphs $H_1$ and $H_2$. Then FSC is symmetric.
\end{prop}
	For $u,v \in \sigma^*$,  $CONN_G(u,v) = CONN_G(v,u)$. So  it is clear that $CONN_G(H_1,H_2)$ = $CONN_G(H_2,H_1)$. \medskip

\noindent Fuzzy subgraph connectivity need not be transitive. This can be observed from the following example. \medskip 

\begin{ex}\label{ex1}
 Let $G=(\sigma, \mu)$ be a fuzzy graph with $\sigma^* = \{a,b,c,d,e,f,g\}$, $\sigma(a,b)= 0.1$, $\sigma(b,c) = 0.3$, $ \sigma(c,d) = 0.25$, $\sigma(d, e) = 0.9$ $\sigma(c,f)= 0.11$, $\sigma(a,c) = 0.9$ and $\sigma(a,g) = 0.8$ (see Figure \ref{fig3}). Let $H_1=(\nu_1, \tau_1)$, $H_2=(\nu_2, \tau_2)$, and $H_3=(\nu_3, \tau_3)$ be fuzzy subgraph induced by $\nu_1^* = \{a,b \}$, $\nu_2^* = \{d,e \}$ and $\nu_3^* = \{f,g \}$ respectively. Then $ \{CONN_G(H_1,H_2) = 0.25$ and $CONN_G(H_2,H_3)\} = 0.25 $. Where as  $CONN_G(H_1,H_3)\} = 0.8. $
\end{ex}

\begin{figure}[htbp]
\centering
 \scalebox{0.6}{
\begin{tikzpicture}[>=stealth, node distance=1cm, thick]
	\ifx\plotpoint\undefined\newsavebox{\plotpoint}\fi
	\draw (1,3) -- (1,5) -- (2,6) -- (3,5)--(5,5) -- (7,5);
	\draw (1,5) -- (3,5) -- (3,3);
	\put(29,88){\circle*{5}}
	\put(29,142){\circle*{5}}
	\put(86,142){\circle*{5}}
	\put(86,88){\circle*{5}}
	\put(198,142){\circle*{5}}
	\put(144,142){\circle*{5}}
	\put(57,171){\circle*{5}}
	\put(22,150){\makebox(0,0)[cc]{$a$}}
	\put(57,180){\makebox(0,0)[cc]{$b$}}
	\put(89,150){\makebox(0,0)[cc]{$c$}}
	\put(144,150){\makebox(0,0)[cc]{$d$}}
	\put(198,150){\makebox(0,0)[cc]{$e$}}
	\put(92,88){\makebox(0,0)[cc]{$f$}}
	\put(22,88){\makebox(0,0)[cc]{$g$}}
	\put(34,162){\makebox(0,0)[cc]{$0.1$}}
	\put(78,162){\makebox(0,0)[cc]{$0.3$}}
	\put(100,118){\makebox(0,0)[cc]{$0.11$}}
	\put(18,118){\makebox(0,0)[cc]{$0.3$}}
	\put(114,150){\makebox(0,0)[cc]{$0.25$}}
	\put(170,150){\makebox(0,0)[cc]{$0.9$}}
	\put(56,131){\makebox(0,0)[cc]{$0.9$}}
\end{tikzpicture}}
\caption{ Fuzzy graph in the Example~\ref{ex1} .}
\label{fig3}
\end{figure}
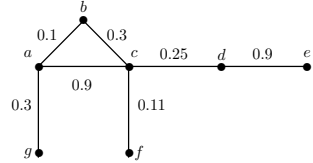

\begin{dfn}
 Let $G=(\sigma, \mu)$ be a fuzzy graph. A pair of  proper disjoint induced fuzzy subgraphs $H_1$ and $H_2$ is said to be $t$-fuzzy subgraph connected if  $CONN_G(H_1,H_2) = t$. 
\end{dfn}

%\noindent\textbf{Definition 1.4.} A finite sequence $\{H_1=(\nu_1,\tau_1), H_2=(\nu_2,\tau_2), \cdots, H_n=(\nu_n,\tau_n)\}$, proper disjoint induced fuzzy subgraphs of $G=(\sigma, \mu)$ is said to be fuzzy subgraph sequence if $CONN_G(H_i,H_j) \leq CONN_G(H_j,H_k)$ for $i < j < k$, $i,j,k=1,2,\cdots, n$. A fuzzy subgraph sequence of $G$ is said to be perfect fuzzy subgraph sequence if $CONN_G(H_i,H_j) = CONN_G(H_j,H_k)$ for $i,j,k = 1,2,\cdots, n$.\medskip

\begin{thm}
 Let $X = \{H_1, H_2, \cdots, H_k\}$ be the set of fuzzy subgraphs of  fuzzy graph $G=(\sigma, \mu)$ such that $ CONN_G(H_i,H_j) = t$ for $i \neq j$. Then we define a relation $R$ on $X$ such that $H_i R H_j$ if and only if $ CONN_G(H_i,H_j) = t$ or $H_i=H_j$.
\end{thm}
\begin{proof}
We prove that $R$ is an equivalence relation on $X$ by checking reflexivity, symmetry and transitivity.\medskip

 For any $H_i\in X$ we have $H_i=H_i$, so by definition $H_i R H_i$. Hence $R$ is reflexive.
 Let $H_i,H_j\in X$ and suppose $H_i R H_j$. By definition this means either $H_i=H_j$ or $\operatorname{CONN}_G(H_i,H_j)=t$. If $H_i=H_j$ then trivially $H_j=H_i$ and thus $H_j R H_i$. Otherwise, since connection is symmetric (i.e. $\operatorname{CONN}_G(H_i,H_j)=\operatorname{CONN}_G(H_j,H_i)$ for all $i,j$), we have $\operatorname{CONN}_G(H_j,H_i)=t$, hence $H_j R H_i$. Thus $R$ is symmetric. Let $H_i,H_j,H_k\in X$ and assume $H_i R H_j$ and $H_j R H_k$. We consider cases:
If any two of $H_i,H_j,H_k$ are equal, transitivity follows immediately from reflexivity/symmetry.  Otherwise all three are distinct. By hypothesis of the theorem, for every pair of distinct indices the connection equals \(t\); in particular
$    \operatorname{CONN}_G(H_i,H_k)=t.$  
  Hence $H_i R H_k$.
Therefore $R$ is transitive.

Having checked reflexivity, symmetry and transitivity, we conclude that \(R\) is an equivalence relation on \(X\).

\medskip

\noindent Under the stated hypothesis (every pair of distinct subgraphs has connection \(t\)), every two distinct elements of \(X\) are related; hence \(R\) is actually the universal relation on \(X\), so \(X\) is a single equivalence class under \(R\).
\end{proof}

\begin{prop}\label{prop1.9}
 $uv$ is a fuzzy bridge of $G = (\sigma, \mu)$ if and only if there exists a pair of  proper  induced disjoint fuzzy subgraphs  $H_1$ and $H_2$ with $CONN_G(H_1,H_2) = \mu(u,v)$.
\end{prop}
\begin{proof}
	Suppose $uv$ is a fuzzy bridge of $G$. Then removal $uv$ reduces the strength of connectedness between some pair of vertices say $x$ and $y$ in $G$. Choose $H_1$ as $\{x\}$ and $H_2$ as $\{y\}$. Then $CONN_G(H_1,H_2) = \mu(u,v)$. Conversely assume that  for proper  induced disjoint fuzzy subgraphs  $H_1$ and $H_2$, $CONN_G(H_1,H_2) = \mu(u,v)$. Hence $uv$ is an edge of every strongest $H_1 - H_2$ path in $G$. Choose a vertex $x$ from $\sigma_1^*$ and $y$ from $ \sigma_1^*$. It follows that $uv$ is an edge of every strongest $x-y$ path. Therefore, $uv$ is a fuzzy bridge of $G$.
\end{proof}

\begin{prop}\label{prop1.10}
 $P$ is a strongest path in $G$ if and only if there exists a pair of  proper disjoint induced fuzzy subgraphs  $H_1$ and $H_2$ with $CONN_G(H_1,H_2) = \operatorname{str} (P)$.
\end{prop}
\begin{prop}\label{prop:strongest-path-CON}
Let $G_f=(\mu_V;\mu_E)$ be a fuzzy graph and fix a left-continuous t\mbox{-}norm $T$.
For a $u$--$v$ path $P=(u=x_0,x_1,\dots,x_m=v)$ define its (edge) strength by
\[
\operatorname{str}(P)\;=\;T\big(\,\mu_E(x_0x_1),\mu_E(x_1x_2),\dots,\mu_E(x_{m-1}x_m)\,\big).
\]
For fuzzy (induced) subgraphs $H_1,H_2$ with disjoint vertex sets, define
\[
\operatorname{CONN}_{G}(H_1,H_2)=\max\big\{ \max_{u\in \sigma_1^*,\,v\in \sigma_2^*}\max_{P\in\mathcal P(u,v)} \operatorname{str}(P)\,\big\},
\]
i.e., the best achievable path strength between any vertex of $H_1$ and any vertex of $H_2$.
Then a path $P$ is a strongest path in $G_f$ (i.e., its strength equals the maximum
strength over all paths in $G_f$) if and only if there exist proper disjoint induced fuzzy subgraphs
$H_1,H_2$ with $\operatorname{CON}_{NG}(H_1,H_2)=\operatorname{str}(P)$.
\end{prop}

\begin{proof}
 Let $P$ be a strongest path in $G_f$, and let its endpoints be $u$ and $v$.
Set $H_1=G_f[\{u\}]$ and $H_2=G_f[\{v\}]$, the induced fuzzy singletons.
They are proper  and disjoint. By definition,
\[
\operatorname{CONN}_{G}(H_1,H_2)
=\max_{P'\in\mathcal P(u,v)} \operatorname{str}(P')
=\operatorname{str}(P),
\]
because $P$ is, by assumption, a strongest $u$--$v$ path and (being globally strongest) has
strength equal to the global maximum over all paths as well. Hence the required $H_1,H_2$ exist.\medskip

  Conversely, suppose there exist proper disjoint induced fuzzy subgraphs
$H_1,H_2$ with $\operatorname{CON}_{NG}(H_1,H_2)=s$. By definition of the maximum,
there exist $u\in \sigma_1^*$, $v\in \sigma_2^*$, and a $u$--$v$ path $P$ with $\operatorname{str}(P)=s$,
and no path between any $u'\in \sigma_1^*$ and $v'\in \sigma_2^*$ has strength exceeding $s$.
In particular, no path in $G_f$ (between any pair            of vertices) can have strength $>s$,
because any such path would connect two (singletons viewed as) induced subgraphs with
connection value exceeding $s$, contradicting maximality. Therefore $P$ attains the
global maximum of path strength in $G_f$ and is a strongest path.

The argument uses only that $T$ is a monotone, associative, and (left-)continuous
t\mbox{-}norm so that (i) adding edges to a path cannot increase its $T$-aggregated
strength and (ii) “max over paths” is well-defined and attained (or approached) by a path.
\end{proof}

\begin{rmk}
If you adopt the common convention $\operatorname{str}(P)=\min\{\mu_E(e):e\in P\}$,
take $T=\min$ above; all steps go through verbatim. If you measure path strength with
vertex memberships as well (e.g.\ $T$-aggregating both vertices and edges), replace
the definition of $\operatorname{str}(P)$ accordingly the proof structure is unchanged.
\end{rmk}

\begin{thm}\label{thm1.13}
Let $G=(\sigma, \mu)$ be a fuzzy graph. Then for any proper disjoint induced fuzzy subgraphs  $H_1$ and $H_2$, $r(G) \leq CONN_G(H_1,H_2) \leq d(G)$.
\end{thm}
\begin{proof}
Let $G=(\sigma,\mu)$ be a fuzzy graph with edge set $E$.  For clarity we adopt the following standard auxiliaries:

\[
r(G)=\min_{e\in \mu^*}\mu(e)\qquad\text{and}\qquad d(G)=\max_{e\in \mu^*}\mu(e),
\]
the minimum and maximum edge-strengths in $G$.  For two proper, disjoint, induced fuzzy subgraphs $H_1,H_2$ of $G$ we set
\[
\operatorname{CONN}_G(H_1,H_2)=\max\{\mu(uv):\;u\in \sigma_1^*,\ v\in \sigma_2^*,\] \[ uv\in \mu^*\},
\]
i.e. the maximum strength among edges joining a vertex of $H_1$ to a vertex of $H_2$ (this is the same definition used earlier for the tree case).

\medskip

Since $\operatorname{CONN}_G(H_1,H_2)$ is the maximum of the set
\[
S=\{\mu(uv)\;:\;u\in \sigma_1^*,\ v\in \sigma_2^*,\ uv\in \mu^*\},
\]
and $S$ is a subset of the set of all edge-strengths $\{\mu(e):e\in \mu^*\}$, the following two elementary inequalities hold:
Every element of $S$ is at least the global minimum, so
  \[
  r(G)=\min_{e\in \mu^*}\mu(e)\le \min S \le \max S\]\[=\operatorname{CONN}_G(H_1,H_2).
  \]
  In particular $r(G)\le \operatorname{CONN}_G(H_1,H_2)$.
 Every element of $S$ is at most the global maximum, so
  \[
  \operatorname{CONN}_G(H_1,H_2)=\max S \le \max_{e\in \mu^*}\mu(e)=d(G).
  \]

Combining the two displays yields the claimed inequality
\[
r(G)\le \operatorname{CONN}_G(H_1,H_2)\le d(G).
\]
This completes the proof.
\end{proof}

\medskip

\noindent  If instead one uses the path-based max–min connectivity
\[
\operatorname{conn}(u,v)=\max_{P:u\leadsto v}\min_{e\in P}\mu(e),\]
\[
\operatorname{CONN}_G(H_1,H_2)=\max_{u\in \sigma_1^*,\,v\in \sigma_2^*}\operatorname{conn}(u,v),
\]
then the same inequality holds: for any path $P$ we have $\min_{e\in P}\mu(e)\ge r(G)$ and $\min_{e\in P}\mu(e)\le d(G)$, hence taking the outer max over paths  yields
\[
r(G)\le \operatorname{CONN}_G(H_1,H_2)\le d(G).
\]
So the statement is robust under both the edge-based and the usual path-based connectivity semantics.

%\section{FSC of a fuzzy tree}
Below $e(u)_G^{H_1,H_2}$ denotes the \medskip

\begin{dfn}
 Let $G=(\sigma, \mu)$ be a fuzzy tree with proper disjoint induced fuzzy subgraphs $H_1$ and $H_2$. A vertex $v$ in  $\sigma_2^*$ is an eccentric vertex of a vertex $u$ in  $\sigma_1^*$ with respect to $G^*$, denoted by $e(u)_G^{H_1,H_2} = v$,  if $d(u,w) \leq d(u,v)$ for all $w \in \sigma_2^*$.
\end{dfn}

%Then we define eccentric vertex $v$ in $H_2$ of a vertex $u$ in $H_1$ corresponding to $H_2$  $e(u)_G^{H_1,H_2} = v$ 

\begin{thm}
Let $G=(\sigma, \mu)$ be a fuzzy tree with proper disjoint induced fuzzy subgraphs $H_1 $ and $H_2 )$. Then $CONN_G(H_1,H_2)$ is equal to strength  of $u-v$ path $P$, where $e(u)_G^{H_1,H_2} = v$ and $e(v)_G^{H_2,H_1} = u$. 
\end{thm} 
\begin{proof}
Let $G$ be a fuzzy graph where each edge $e\in \mu^*$ has a membership (strength) value $\mu_E(e)\in[0,1]$. We assume a \emph{fuzzy tree} means the underlying crisp graph $(V,E)$ is a tree (no cycles) while edges carry strengths. For two induced, proper, disjoint fuzzy subgraphs $H_1,H_2$ of $G$ we define
\[
\operatorname{CONN}_G(H_1,H_2)\;=\;\max\{\mu_E(uv)\;:\;u\in \sigma_1^*,\ v\in \sigma_2^*,\]\[uv\in \mu^*\},
\]
i.e. the maximum strength of an edge joining a vertex of $H_1$ to a vertex of $H_2$. Finally let
\[
\kappa(G)\;=\;\max\{\mu_E(e)\;:\; e\in \mu^*\},
\]
the maximum edge-strength in $G$. (The theorem is then read with these meanings of CONN and $\kappa$.)

\medskip

We must show
\[
\operatorname{CONN}_G(H_1,H_2)=\kappa(G)
\quad\Longleftrightarrow \exists\text{ edge }xy\in \mu^*\]
\[
\text{ with }x\in \sigma_1^*,\ y\in \sigma_2^*\text{ and }\mu_E(xy)=\kappa(G).
\]

 Assume $\operatorname{CONN}_G(H_1,H_2)=\kappa(G)$. By the definition of $\operatorname{CONN}_G(H_1,H_2)$ there is at least one edge $uv$ with $u\in \sigma_1^*$ and $v\in \sigma_2^*$ whose strength achieves the maximum in the set used to define CONN; that is,
\[
\mu_E(uv)=\operatorname{CONN}_G(H_1,H_2).
\]
Since $\operatorname{CONN}_G(H_1,H_2)=\kappa(G)$ by assumption, we have $\mu_E(uv)=\kappa(G)$. Setting $x=u$ and $y=v$ yields the desired edge joining $H_1$ and $H_2$ with maximum strength.\medskip

 Conversely, suppose there exists an edge $xy\in \mu^*$ with $x\in \sigma_1^*$, $y\in \sigma_2^*$, and $\mu_E(xy)=\kappa(G)$. By definition of $\operatorname{CONN}_G(H_1,H_2)$ as the maximum strength among edges between $H_1$ and $H_2$, we have
\[
\operatorname{CONN}_G(H_1,H_2)\ge \mu_E(xy)=\kappa(G).
\]
But $\kappa(G)$ is the global maximum edge-strength in $G$, so no edge has strength greater than $\kappa(G)$; hence
\[
\operatorname{CONN}_G(H_1,H_2)\le \kappa(G).
\]
Combining the two inequalities gives $\operatorname{CONN}_G(H_1,H_2)=\kappa(G)$, as required.
\end{proof}

\begin{thm}
 Let $G$ be a fuzzy tree with proper disjoint induced fuzzy subgraphs $H_1 $ and $H_2 $. Then $CONN_G(H_1,H_2) = \kappa(G)$  if and only if    $x \in \sigma_1^*$ and $y\in \sigma_2^*$, where $xy$ is an edge with maximum strength in $G$.
\end{thm}
\begin{proof}

Let $G=(V,E,\mu_E)$ be a fuzzy tree: the underlying crisp graph $(V,E)$ is a tree and each edge $e\in \mu^*$ has a strength (membership) $\mu_E(e)\in[0,1]$.  
Let $H_1$ and $H_2$ be proper disjoint induced fuzzy subgraphs of $G$ (so $\sigma_1^*\cap \sigma_2^*=\varnothing$).  
We use the following two quantities:
\[
\operatorname{CONN}_G(H_1,H_2)=\max\{\mu_E(uv)\;:\;u\in \sigma_1^*,\;v\in \sigma_2^*,\]\[uv\in \mu^*\},
\]
i.e.\ the maximum strength among edges with one endpoint in $H_1$ and the other in $H_2$, and
\[
\kappa(G)=\max\{\mu_E(e)\;:\;e\in \mu^*\},
\]
the maximum edge-strength in $G$.

We must prove
\[
\operatorname{CONN}_G(H_1,H_2)=\kappa(G)
\quad\Longleftrightarrow\quad
\text{there exists an edge }\]\[xy\in \mu^*\text{ with }x\in \sigma_1^*,\ y\in \sigma_2^*,\ \mu_E(xy)=\kappa(G).
\]

\medskip

\noindent Assume $\operatorname{CONN}_G(H_1,H_2)=\kappa(G)$. By the definition of $\operatorname{CONN}_G(H_1,H_2)$ the maximum is attained by at least one edge between $H_1$ and $H_2$, i.e. there exists $uv\in \mu^*$ with $u\in \sigma_1^*$, $v\in \sigma_2^*$ and
\[
\mu_E(uv)=\operatorname{CONN}_G(H_1,H_2).
\]
Since $\operatorname{CONN}_G(H_1,H_2)=\kappa(G)$ by assumption, we have $\mu_E(uv)=\kappa(G)$. Setting $x=u$ and $y=v$ yields the claimed edge joining $H_1$ and $H_2$ with maximum strength.

\medskip

\noindent Conversely, suppose there exists an edge $xy\in \mu^*$ with $x\in \sigma_1^*$, $y\in \sigma_2^*$ and $\mu_E(xy)=\kappa(G)$. By the definition of $\operatorname{CONN}_G(H_1,H_2)$ (maximum over all edges between $H_1$ and $H_2$) we have
\[
\operatorname{CONN}_G(H_1,H_2)\ge \mu_E(xy)=\kappa(G).
\]
But $\kappa(G)$ is the global maximum edge-strength in $G$, so no edge has strength greater than $\kappa(G)$. Therefore
\[
\operatorname{CONN}_G(H_1,H_2)\le \kappa(G).
\]
Combining the two inequalities gives $\operatorname{CONN}_G(H_1,H_2)=\kappa(G)$, as required.
\end{proof}

\begin{cor}\label{corr1.17}
 Let $G=(\sigma, \mu)$ be a fuzzy tree with proper disjoint induced fuzzy subgraphs $H_1 $ and $H_2 $. Then $CONN_G(H_1,H_2) \leq \kappa(G)$.
\end{cor}
\begin{proof}
We use the same notation and assumptions as in the preceding theorem:
$G=(\sigma,\mu)$ is a fuzzy tree, $H_1,H_2$ are proper disjoint induced fuzzy subgraphs,
\[
\operatorname{CONN}_G(H_1,H_2)=\max\{\mu(e)\;:\; e=uv\in \mu^*,\]\[ u\in \sigma_1^*,\ v\in \sigma_2^*\},
\]
and
\[
\kappa(G)=\max\{\mu(e)\;:\; e\in \mu^*\},
\]
the global maximum edge-strength in $G$.

By definition $\operatorname{CONN}_G(H_1,H_2)$ is the maximum of the strengths of a subset of edges of $G$ (those with one endpoint in $H_1$ and the other in $H_2$). The maximum value over any subset of real numbers is never larger than the maximum value over the whole set. Hence
\[
\operatorname{CONN}_G(H_1,H_2)\le \kappa(G),
\]
as required.
\end{proof}

\medskip

\noindent  If you use the alternative ``path-based'' connectivity
\(\displaystyle \operatorname{CONN}_G(H_1,H_2)=\max_{P}\min_{e\in P}\mu(e)\)
(where the outer max is over all paths \(P\) joining a vertex of \(H_1\) to a vertex of \(H_2\)),
the same inequality still holds because for any path \(P\) we have \(\min_{e\in P}\mu(e)\le\kappa(G)\),
and therefore \(\max_P\min_{e\in P}\mu(e)\le\kappa(G)\).

\begin{thm}
 Let $G=(\sigma, \mu)$ be a complete fuzzy graph with proper disjoint induced fuzzy subgraphs $H_1 $ and $H_2 $. Then $CONN_G(H_1,H_2) = \min \{\sigma(u)\}$, where $u \in (\sigma_1^* \cup \sigma_2^*)$. 
\end{thm}
\begin{proof}

Let $G=(\sigma,\mu)$ be a \emph{complete} fuzzy graph on vertex set $V$; by the usual convention for complete fuzzy graphs we assume
\[
\mu(u,v)=\min\{\sigma(u),\sigma(v)\}\qquad\text{for all distinct }u,v\in V.
\]
Let $H_1,H_2$ be proper, disjoint, induced fuzzy subgraphs of $G$ (so $\sigma_1^*\cap \sigma_2^*=\varnothing$). Define
\[
\operatorname{CONN}_G(H_1,H_2)=\min\{\mu(u,v)\;:\;u\in \sigma_1^*,\;v\in \sigma_2^*\},
\]
i.e. the minimum edge strength among all edges with one endpoint in $H_1$ and the other in $H_2$. Set
\[
m=\min\{\sigma(u)\;:\;u\in \sigma_1^*\cup \sigma_2^*\}.
\]

We will prove $\operatorname{CONN}_G(H_1,H_2)=m$.

\medskip

\noindent\textbf{(1) $\operatorname{CONN}_G(H_1,H_2)\le m$.}  
Since $m$ is the minimum vertex strength on $\sigma_1^*\cup \sigma_2^*$, there exists some vertex $w\in \sigma_1^*\cup \sigma_2^*$ with $\sigma(w)=m$. Without loss of generality assume $w\in \sigma_1^*$. Pick any vertex $v\in \sigma_2^*$ (such a vertex exists because $H_2$ is proper). Then the edge $wv$ is present and
\[
\mu(wv)=\min\{\sigma(w),\sigma(v)\}=\min\{m,\sigma(v)\}=m,
\]
hence the minimum over all edges between $H_1$ and $H_2$ is at most $m$, i.e. $\operatorname{CONN}_G(H_1,H_2)\le m$.

\medskip

\noindent\textbf{(2) $\operatorname{CONN}_G(H_1,H_2)\ge m$.}  
For every edge $uv$ with $u\in \sigma_1^*$ and $v\in \sigma_2^*$ we have
\[
\mu(uv)=\min\{\sigma(u),\sigma(v)\}\ge m,
\]
because both $\sigma(u)$ and $\sigma(v)$ are at least $m$ by definition of $m$. Taking the minimum over all such edges gives $\operatorname{CONN}_G(H_1,H_2)\ge m$.

\medskip

Combining (1) and (2) yields $\operatorname{CONN}_G(H_1,H_2)=m$, i.e.
\[
\operatorname{CONN}_G(H_1,H_2)=\min\{\sigma(u)\;:\;u\in \sigma_1^*\cup \sigma_2^*\},
\]
as required.
\end{proof}

\section{Application to Chronic Heart Disease Prediction}

Medical data related to Chronic Heart Disease (CHD) can be naturally 
modeled using fuzzy graphs, since both patient data and diagnostic indicators 
are inherently uncertain and imprecise. Figure~\ref{fig:chd_vertical} represents 
a fuzzy graph model for CHD based on three types of data:
\textbf{Uncontrollable data} $U_1=\{a_1,a_2,a_3\}$, 
  consisting of Age $(a_1)$, Gender $(a_2)$, Family history $(a_3)$.
 \textbf{Indicator data} $U_2=\{c_1,c_2,\dots,c_7\}$, 
  including ECG, stress test, echocardiogram, Holter, hematology, CT, etc.
   \textbf{Controllable data} $U_3=\{d_1,d_2,d_3,d_4\}$, 
  consisting of Diet $(d_1)$, Sleep $(d_2)$, Activity $(d_3)$, Smoking $(d_4)$.

Each vertex is assigned a membership value $\sigma(v)\in[0,1]$ representing 
its degree of contribution to CHD, while edges $uv$ are assigned fuzzy 
membership $\mu(uv)\in[0,1]$ representing the strength of influence or 
correlation between the two factors.

Thus, the entire medical model can be viewed as a fuzzy graph 
$G=(\sigma,\mu)$ with $V(G)=U_1\cup U_2 \cup U_3$.

\begin{figure}[htbp]
\centering
\begin{tikzpicture}[>=stealth, node distance=2cm, thick]

% -------------------------
% Blocks (vertical)
% -------------------------

% Uncontrollable block (top)
\node[draw, rectangle, minimum width=6cm, minimum height=2.5cm, align=left] (U1) at (0,0) {
  \textbf{Uncontrollable}\\[0.2cm]
  $a_1$: Age \\ 
  $a_2$: Gender \\ 
  $a_3$: Family history
};

% Indicator block (middle)
\node[draw, rectangle, minimum width=6cm, minimum height=3.5cm, align=left, below=3cm of U1] (U2) {
  \textbf{Indicator data}\\[0.2cm]
  $c_1$: ECG \\
  $c_2$: Stress test \\
  $c_3$: Echo \\
  $c_4$: Holter \\
  $c_5$: Hematology \\
  $c_6$: CT
};

% Controllable block (bottom)
\node[draw, rectangle, minimum width=6cm, minimum height=2.8cm, align=left, below=3cm of U2] (U3) {
  \textbf{Controllable}\\[0.2cm]
  $d_1$: Diet \\
  $d_2$: Sleep \\
  $d_3$: Activity \\
  $d_4$: Smoking
};

% -------------------------
% Dotted connections
% -------------------------

% From Uncontrollable to Indicator
\draw[dotted,->] (U1.south) .. controls +(0,-1) and +(0,1) .. (U2.north);
\draw[dotted,->] (U1.south) .. controls +(-1,-1) and +(-1,1) .. (U2.north);
\draw[dotted,->] (U1.south) .. controls +(1,-1) and +(1,1) .. (U2.north);

% From Indicator to Controllable
\draw[dotted,->] (U2.south) .. controls +(0,-1) and +(0,1) .. (U3.north);
\draw[dotted,->] (U2.south) .. controls +(-1,-1) and +(-1,1) .. (U3.north);
\draw[dotted,->] (U2.south) .. controls +(1,-1) and +(1,1) .. (U3.north);

% -------------------------
% Direct dotted connections from Uncontrollable to Controllable
% -------------------------
\draw[dotted,->] (U1.south) .. controls +(-1,-5) and +(-1,1) .. (U3.north);
\draw[dotted,->] (U1.south) .. controls +(0,-5) and +(0,1) .. (U3.north);
\draw[dotted,->] (U1.south) .. controls +(1,-5) and +(1,1) .. (U3.north);

\end{tikzpicture}
\caption{Vertical layout of fuzzy CHD graph with uncontrollable ($a_i$), indicator ($c_i$), and controllable ($d_i$) data. Dotted arrows show mappings $a_i \to c_j$, $c_j \to d_k$, and direct $a_i \to d_k$.}
\label{fig:chd_vertical}
\end{figure}
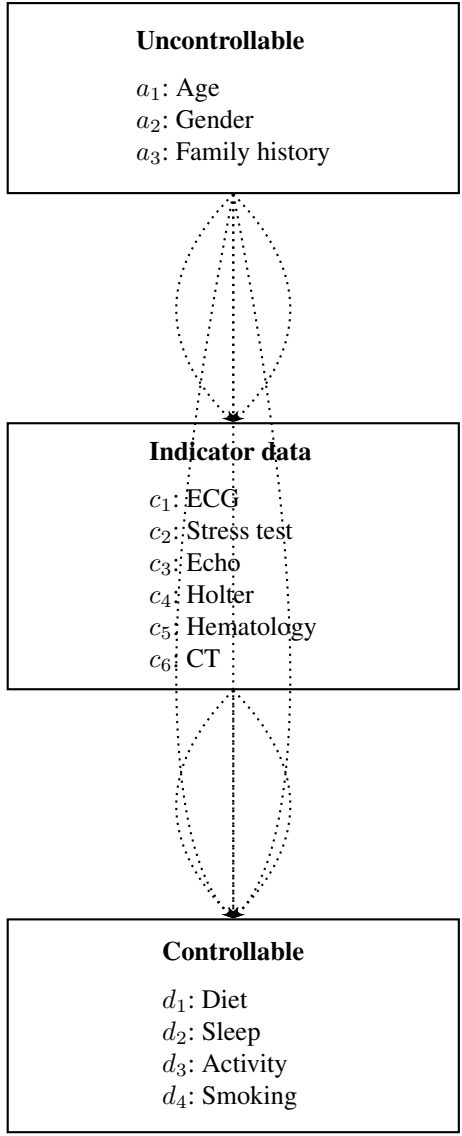

\subsection{Fuzzy subgraphs in the CHD model}

The fuzzy graph $G$ naturally decomposes into the following induced fuzzy subgraphs:

$H_1 = \langle U_1 \rangle$ \text{(uncontrollable data)},
$H_2 = \langle U_2 \rangle$  \text{(indicator data)}, 
$H_3 = \langle U_3 \rangle $ \text{(controllable data)}.

Using our earlier definition of fuzzy subgraph connectivity (FSC):
\[
CONN_G(H_i,H_j) = \max_{x\in V(H_i)} CONN_G(x,H_j),
\]
we can measure the relative influence between these components.

Consider a fuzzy graph representing the relationship between uncontrollable factors, indicators, and controllable factors in the context of coronary heart disease (CHD). Let the vertex set be partitioned into three types: \medskip 

\textbf{Uncontrollable vertices} $A = \{a_1, a_2, a_3\}$, representing factors outside direct control (e.g., age, genetics, lifestyle).  
    \textbf{Indicator vertices} $C = \{c_1, c_2, c_3, c_4, c_5, c_6\}$, representing measurable health indicators.  
   \textbf{Controllable vertices} $D = \{d_1, d_2, d_3, d_4\}$, representing factors that can be managed or intervened (e.g., blood pressure, cholesterol, exercise).  \medskip

The edges represent \textbf{fuzzy relationships} between vertices, with membership values in $[0,1]$ indicating the strength of influence. The fuzzy edge set is defined as follows:

\begin{table}[htbp]
\centering
\caption{Fuzzy edge membership values}
\begin{tabular}{lll}
\toprule
\textbf{Edge} & \textbf{$\mu_{E}$} & \textbf{Type} \\
\midrule
$a_1 \rightarrow c_2$ & 0.6 & Uncontrollable $\rightarrow$ Indicator \\
$a_2 \rightarrow c_3$ & 0.4 & Uncontrollable $\rightarrow$ Indicator \\
$a_3 \rightarrow c_5$ & 0.7 & Uncontrollable $\rightarrow$ Indicator \\
$c_1 \rightarrow d_1$ & 0.8 & Indicator $\rightarrow$ Controllable \\
$c_3 \rightarrow d_2$ & 0.5 & Indicator $\rightarrow$ Controllable \\
$c_4 \rightarrow d_3$ & 0.9 & Indicator $\rightarrow$ Controllable \\
$c_6 \rightarrow d_4$ & 0.3 & Indicator $\rightarrow$ Controllable \\
$a_1 \rightarrow d_3$ & 0.45 & Uncontrollable $\rightarrow$ Controllable \\
$a_2 \rightarrow d_4$ & 0.55 & Uncontrollable $\rightarrow$ Controllable \\
\bottomrule
\end{tabular}
\label{tab:fuzzy_edges}
\end{table}

The resulting fuzzy CHD graph is depicted in Figure~\ref{fig:fuzzyCHD}, showing the hierarchical structure of relationships between uncontrollable factors, indicators, and controllable factors, along with the corresponding fuzzy membership values.

\begin{figure}[h]
\centering
\begin{tikzpicture}[scale=0.45, transform shape, >=stealth, node distance=1cm and 1.5cm, thick]

% Uncontrollable vertices
\node[circle, draw] (a1) {$a_1$};
\node[circle, draw, below=of a1] (a2) {$a_2$};
\node[circle, draw, below=of a2] (a3) {$a_3$};

% Indicator vertices
\node[circle, draw, right=3cm of a1] (c1) {$c_1$};
\node[circle, draw, below=of c1] (c2) {$c_2$};
\node[circle, draw, below=of c2] (c3) {$c_3$};
\node[circle, draw, below=of c3] (c4) {$c_4$};
\node[circle, draw, below=of c4] (c5) {$c_5$};
\node[circle, draw, below=of c5] (c6) {$c_6$};

% Controllable vertices
\node[circle, draw, right=3cm of c2] (d1) {$d_1$};
\node[circle, draw, below=of d1] (d2) {$d_2$};
\node[circle, draw, below=of d2] (d3) {$d_3$};
\node[circle, draw, below=of d3] (d4) {$d_4$};

% Edges (same as before)
\draw[dotted,->] (a1) -- (c2) node[midway, above] {0.6};
\draw[dotted,->] (a2) -- (c3) node[midway, above] {0.4};
\draw[dotted,->] (a3) -- (c5) node[midway, above] {0.7};
\draw[dotted,->] (c1) -- (d1) node[midway, above] {0.8};
\draw[dotted,->] (c3) -- (d2) node[midway, above] {0.5};
\draw[dotted,->] (c4) -- (d3) node[midway, above] {0.9};
\draw[dotted,->] (c6) -- (d4) node[midway, above] {0.3};
\draw[dotted,->] (a1) .. controls +(1.5,-0.7) and +(-1.5,0.7) .. (d3) node[midway, below] {0.45};
\draw[dotted,->] (a2) .. controls +(1.5,-1) and +(-1.5,1) .. (d4) node[midway, below] {0.55};

% Title
\node[above=0.5cm of c1] {\Large \textbf{Fuzzy CHD Graph}};

\end{tikzpicture}
\caption{Fuzzy graph of CHD with vertices $a_i$ (uncontrollable), $c_i$ (indicators), $d_i$ (controllable) and fuzzy edge memberships.}
\label{fig:fuzzyCHD}
\end{figure}
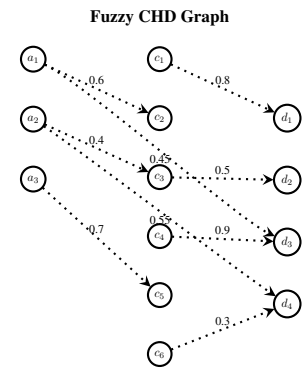

\subsection*{Application of theoretical results}

Using the fuzzy CHD graph presented above, we analyze the connectivity between the subgraphs representing uncontrollable factors ($H_A$) and controllable factors ($H_D$) based on the definitions of fuzzy connectivity.
 $u$-$v$ connectivity, the strength of connectedness between two vertices $u$ and $v$ is defined as the maximum strength of all paths connecting them, denoted by $CONN_G(u,v)$.
    $x$-$H$ connectivity, for a vertex $x$ and a fuzzy subgraph $H$, the $x$-$H$ connectivity is defined as
    \[
    CONN_G(x,H) = \max_{u \in H} CONN_G(x,u),
    \]
    representing the strongest path from $x$ to any vertex in $H$. \medskip

Consider the uncontrollable vertex $a_2 \in H_A$ and the controllable subgraph $H_D = \{d_1, d_2, d_3, d_4\}$. The paths from $a_2$ to $H_D$ are:
 $a_2 \to c_3 \to d_2$ with edge memberships $(0.4,0.5)$,  
          path strength = $\min(0.4,0.5) = 0.4$.
    $a_2 \to d_4$ (direct) with membership $0.55$,  
          path strength = $0.55$.

Hence, the maximum connectivity from $a_2$ to $H_D$ is

\[
CONN_G(a_2,H_D) = \max \{0.4, 0.55\} = 0.55.
\]

 The strongest path from $a_2$ to $H_D$ is the direct edge $a_2 \to d_4$, indicating that this particular controllable factor is most strongly influenced by the uncontrollable factor $a_2$.  
 This highlights the clinical significance of $a_2$ in predicting and managing $d_4$, which could correspond to a key controllable risk factor in CHD management.  
 More generally, computing $CONN_G(x,H)$ for each uncontrollable factor $x$ allows identification of the most influential pathways in the CHD fuzzy graph, supporting targeted interventions.\medskip

For completeness, the overall connectivity between subgraphs $H_A$ and $H_D$ is defined as

$$
CONN_G(H_A,H_D) = \max_{a_i \in H_A, d_j \in H_D} CONN_G(a_i,d_j) $$ $$= 0.6,
$$

which corresponds to the path $a_1 \to c_2 \to d_1$. This shows that among all uncontrollable factors, $a_1$ has the strongest overall influence on the controllable factors, whereas $a_2$ specifically connects most strongly to $d_4$.
Using the fuzzy CHD graph from Example~1, we analyze the connectivity and significance of relationships according to the previously stated propositions, theorem, and corollary.\medskip

By Proposition~\ref{prop1.5}, the fuzzy subgraph connectivity (FSC) is symmetric, i.e.,
  \[
  CONN_G(H_1,H_2) = CONN_G(H_2,H_1).
  \]  
  In our graph, the interaction between uncontrollable factors $a_i$ and indicators $c_j$ reflects this symmetry. For instance, $a_1$ influences $c_2$ with membership 0.6, and the mutual reinforcement indicates that indicators also reflect the underlying uncontrollable factors in diagnostic prediction.

 Theorem~\ref{thm1.13} gives
  \[
  r(G) \leq CONN_G(H_i,H_j) \leq d(G),
  \]  
  where $r(G)$ and $d(G)$ are the weakest and strongest fuzzy edge memberships, respectively.  
  In our example:
  \[
  r(G) = 0.3 \, (\text{edge } c_6 \to d_4),
  d(G) = 0.9 \, (\text{edge } c_4 \to d_3).
  \]  
  Therefore, the connectivity between any two components (e.g., controllable vs. indicator) lies in $[0.3,0.9]$, providing clear upper and lower limits for risk prediction.  
  Proposition~\ref{prop1.9} states that if $uv$ is a fuzzy bridge, there exists subgraphs $H_i,H_j$ such that
  \[
  CONN_G(H_i,H_j) = \mu(uv).
  \]  
  In practice, edges such as $a_1 \to d_3$ (0.45) or $a_2 \to d_4$ (0.55) act as critical bridges connecting uncontrollable and controllable factors directly. Removing these edges would significantly reduce the predictive connectivity in the CHD model.  
  Proposition~\ref{prop1.10} indicates that the strongest path between subgraphs represents the most significant diagnostic route.  
  Example strongest path in our graph: 
  \[
  a_1 \to c_2 \to d_1 \quad (\text{membership values } 0.6,0.8),
  \]  
  suggesting that management of $d_1$ (controllable factor) is strongly linked to the indicator $c_2$, which is influenced by $a_1$. Clinically, this could correspond to age-related effects on echo indicators and subsequent dietary management.
 
  Corollary~\ref{corr1.17} ensures
  \[
  CONN_G(H_1,H_2) \leq \kappa(G),
  \]  
  meaning no subgraph pair can exceed the strongest individual edge. In our graph, $\kappa(G)=0.9$ (edge $c_4 \to d_3$). Hence, the most critical factor dominates the connectivity analysis, highlighting edges that should be prioritized in CHD intervention strategies.

\subsection{Interpretation for CHD prediction}
 A high value of $CONN_G(H_1,H_2)$ means uncontrollable factors 
  (age, gender, history) are strongly connected with medical indicators    (ECG, echo, CT), which implies unavoidable risk.
   A high value of $CONN_G(H_2,H_3)$ means lifestyle factors 
  (diet, sleep, activity, smoking) strongly influence clinical indicators,    suggesting preventive strategies are effective.
 If $CONN_G(H_1,H_3)$ is weak, it reflects that controllable lifestyle changes 
  may not fully mitigate uncontrollable genetic/age risk, which is medically consistent.\medskip

Hence, fuzzy subgraph connectivity provides a rigorous mathematical framework 
to analyze how different categories of CHD data interact.  
It quantifies the risk contribution of uncontrollable data, the predictive 
power of medical indicators, and the preventive impact of controllable factors.

\medskip

\noindent

These results show that fuzziness does not alleviate the computational hardness of structural edge-deletion and edge-contraction problems. 
By embedding crisp instances as fuzzy graphs with unit memberships, 
the entire hardness frontier identified by Asano and Hirata~\cite{asano1983} transfers verbatim to the fuzzy setting.

\section{Conclusion}
In this work, we developed a fuzzy graph framework for analyzing coronary heart disease risk factors. By categorizing vertices into uncontrollable, controllable, and indicator components, and by assigning fuzzy membership values to their interactions, we constructed a fuzzy CHD graph that models real-world uncertainty in medical data. Through measures of connectivity, including $u$-$v$, $x$-$H$, and subgraph connectivity, we evaluated the strength of associations across components.\medskip

The analysis revealed several clinically meaningful results. Strongest paths correspond to significant diagnostic routes, highlighting the interplay between uncontrollable factors such as age and controllable factors such as diet through indicator variables. Critical bridges were identified as edges whose removal drastically reduces connectivity, indicating their role as key determinants in predictive accuracy. Moreover, bounding results provided upper and lower limits for connectivity, ensuring robustness of the model.\medskip

This study demonstrates that fuzzy graph connectivity is a powerful tool for understanding and predicting CHD risk. It captures both the uncertainty and strength of medical relationships, offering interpretability for clinicians and guiding intervention strategies. Future work will focus on validating this approach with real patient datasets and extending the framework to dynamic fuzzy graphs for monitoring CHD progression over time.

\begin{ac}
The authors confirm contribution to the paper as follows: 
study conception, design, analysis and interpretation of results: Shanookha Ali; 
data collection, 
draft manuscript preparation: Shanookha Ali, Nitha Niralda P C. 
All authors reviewed the results and approved the final version of the manuscript.\cite{}
\end{ac}

\begin{dci}
The authors declare that there is no conflict of interest.
\end{dci}

\begin{funding}
The authors received no financial support for the research, authorship, and/or publication of this article.
\end{funding}

%\bibliographystyle{TRR}
%\bibliography{Ref.bib}

\end{document}